%% file: ms.tex
\relax
\pdfoutput=1
\documentclass{article}
\usepackage{multirow}
\usepackage{arxiv}  
\usepackage{times}  
\usepackage{helvet} 
\usepackage{courier}  
\usepackage[hyphens]{url}  
\usepackage{graphicx} 
\usepackage{amsfonts}
\usepackage{amsmath}
\usepackage{amsthm}
\usepackage{cite}
\usepackage[symbol]{footmisc}

\usepackage{caption}
\usepackage{subcaption}
\usepackage[switch]{lineno}  
 
\usepackage{graphicx}  
\title{CounteRGAN: Generating Realistic Counterfactuals with Residual Generative Adversarial Nets}

\author{Daniel Nemirovsky\footnotemark \\
  Amazon\\
  Seattle, WA\\
  U.S.A.\\
  \texttt{nemird@amazon.com} \\
  \And
Nicolas Thiebaut \\ 
  Hired\\
  San Francisco, CA\\
  U.S.A.\\
  \texttt{nicolas@hired.com} \\
  \And
Ye Xu\textsuperscript{*}\\
  Facebook\\
  Menlo Park, CA\\
  U.S.A.\\
  \texttt{yexu@fb.com} \\
  \And
Abhishek Gupta\textsuperscript{*}\\
  Facebook\\
  Menlo Park, CA\\
  U.S.A.\\
  \texttt{abigupta@fb.com} \\
  }

\date{}

\newtheorem{prop}{Proposition}

\begin{document}

\maketitle


\input{abstract}

\input{intro}

\input{related_work}

\input{countergan}
\input{experiments}

\input{discussion}

\newpage
\input{ethical_statement}
\bibliography{rgan_counterfactuals.bib}
\bibliographystyle{unsrt} 
\end{document}

%% file: abstract.tex
\begin{abstract}
The prevalence of machine learning models in various industries has led to growing demands for model interpretability and for the ability to provide meaningful recourse to users. For example, patients hoping to improve their diagnoses or loan applicants seeking to increase their chances of approval. Counterfactuals can help in this regard by identifying input perturbations that would result in more desirable prediction outcomes. Meaningful counterfactuals should be able to achieve the desired outcome, but also be realistic, actionable, and efficient to compute. Current approaches achieve desired outcomes with moderate actionability but are severely limited in terms of realism and latency. To tackle these limitations, we apply Generative Adversarial Nets (GANs) toward counterfactual search. We also introduce a novel Residual GAN (RGAN) that helps to improve counterfactual realism and actionability compared to regular GANs. The proposed CounteRGAN method utilizes an RGAN and a target classifier to produce counterfactuals capable of providing meaningful recourse. Evaluations on two popular datasets highlight how the CounteRGAN is able to overcome the limitations of existing methods, including latency improvements of $>$50x to $>$90,000x, making meaningful recourse available in real-time and applicable to a wide range of domains.
\end{abstract}

%% file: intro.tex
\section{Introduction \label{sec:intro}}
\noindent

\renewcommand{\thefootnote}{\fnsymbol{footnote}}
\footnotetext[1]{Work done while at Hired.}
\addtocounter{footnote}{-1}
\renewcommand{\thefootnote}{\arabic{footnote}}

\begin{figure}[t]
\centering
\includegraphics[width=0.7\columnwidth]{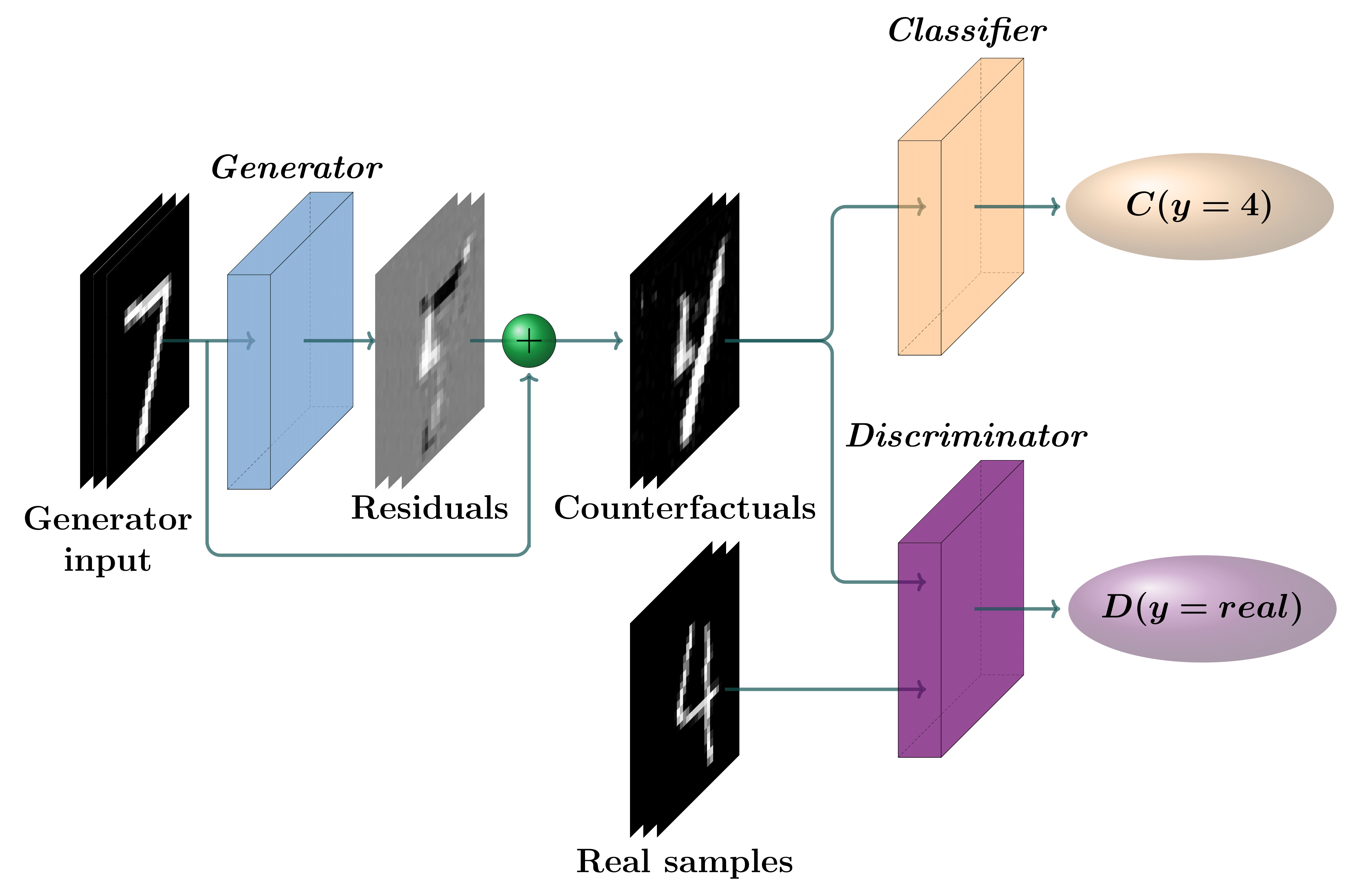}
\caption{The CounteRGAN method on an example from MNIST. Three neural networks are used, a generator trained to output residuals, a discriminator trained to distinguish realistic data, and a target classifier. The model's loss function uses the output from both the classifier and discriminator. The example shows a generator outputting residuals that, when added with the input, produces realistic counterfactual images of a "4".}
\label{fig:countergan_arch}
\end{figure}



Machine learning (ML) predictive models have been widely applied to provide instrumental information in daily scenarios. For instance, in healthcare to diagnose patients \cite{Miotto2018-ah}, finance for loan approvals \cite{Addo2018-gl}, recruiting for matching candidates to jobs \cite{Faliagka2012-pg}, and criminal justice for predicting recidivism \cite{Tollenaar2013-no}. This pervasiveness has resulted in a growing demand for model interpretability as well as influential discussions regarding the "right to explanation" in the machine learning and legal communities \cite{Wachter2017-vk, Selbst2017-zd, Goodman2016-co}. As a result, considerable effort has been made to develop interpretability methods targeted not just towards prediction explainability, but also for improving fairness and opportunity by providing recourse to users.

Leading explainability methods \cite{Ribeiro2016-fa, Lundberg2017-ly, Sundararajan2017-cq, Selvaraju_2019, Chattopadhay_2018} have shown great promise in illuminating the often opaque logic and feature influences behind a model's prediction. By answering the question of \textit{why} a model predicted the outcome it did, explainability methods are useful for validating training and highlighting problematic biases related to sensitive issues such as wealth, race, and gender. Recourse, by contrast, aims at providing interpretable and actionable feedback to users. It helps to answer \textit{how} a prediction can be altered or reversed. It does so by suggesting certain changes, or perturbations, to make to the input data. Recourse specifies, for example, what changes an individual should make to improve their chance of receiving a better medical diagnosis or being approved for a loan or job interview. 

Alternative scenarios, or hypotheticals, that rely on perturbations to the original input values are known as \textit{counterfactuals}. The impacts on a prediction from the changes suggested by counterfactuals can be useful for model interpretability as well as for providing recourse to users. For instance, predictor biases could be detected if the counterfactual suggests changing one's gender or race to alter the prediction result. Recourse, on the other hand, requires providing interpretable feedback that is reasonable for users to act upon and which would help change the prediction result in their favor. To enable recourse, it is imperative that the counterfactuals be \textit{meaningful counterfactuals}, defined as being \textit{realistic}, \textit{computationally efficient} to generate, and able to provide \textit{actionable} feedback to the user that would help \textit{achieve the desired prediction} outcome given a target predictor.\footnote{A pedagogical example of counterfactual search using a synthetic dataset is provided in the supplementary material.} 

Whereas computational efficiency and prediction outcome can be easily quantified using latency and predictor score, realism and actionability are more subjective in nature. Realism relates to how well the counterfactual resembles or fits-in with the known data distribution. For example, a house with a negative number of bedrooms is evidently unrealistic. A less obvious example, however, would be of a house with a seemingly extreme layout but where realism is dependent upon the location and society.\footnote{One of the authors recalls the wonderment of seeing the tall and narrow Dutch houses neatly packed into picturesque rows lining idyllic canals. Consider how surreal such homes would appear in the Andean mountain villages or vice versa.} Actionability, on the other hand, pertains to whether suggested changes are interpretable and reasonable for a user to act upon. Improving one's body mass index, learning a new programming language, or reducing outstanding debt are examples of actionable changes (granted some are harder than others). Proximity and sparsity can serve as intuitive yet imperfect proxies for actionability since they represent the magnitude and number of perturbations suggested by a counterfactual. 
A realistic counterfactual may not always lead to actionable changes. For instance, it is not reasonable to reduce one's age or education even though this may lead to a counterfactual which could very well describe a realistic individual. Moreover, depending on the use case, actionable changes may result in unrealistic counterfactuals. For instance, manipulating pixels and text or fixing features to specific values can confuse the target classifier in a manner similar to adversarial attacks which exploit seldom used regions of a classifier's decision boundary. 

Existing recourse methods \cite{Wachter2017-jr, Van_Looveren2019-hr, Mothilal2020-ge} employ variations of regularized gradient descent to perform the counterfactual search. This method is acutely latency constrained due to having to perform a separate counterfactual search for each unique input data point. Low counterfactual realism also hinders approaches that do not include explicit realism constraints in their algorithm \cite{Wachter2017-jr} or conflate realism with actionability \cite{Mothilal2020-ge}. The latency constraints and distinction between realism and actionability are crucial for framing the counterfactual search problem as a natural fit for Generative Adversarial Networks (GANs). GANs\cite{Goodfellow2014-wf} are a class of ML models capable of producing strikingly realistic synthetic data with low and fixed latencies.
These models formulate the training of two artificial neural networks, a generator and a discriminator, as an adversarial game. The discriminator is trained to distinguish realistic data while the generator aims to synthesize data that is able to fool the discriminator. An effectively trained generator will be able to produce realistic data requiring only a forward-pass through the neural network.




In this work, we formalize a \textit{Residual GAN (RGAN)} architecture, useful for generating perturbations directly and alleviating mode collapse. This later issue is a common training breakdown where the generator consistently produces identical or similar outputs, regardless of the inputs. The RGAN is used in conjunction with a fixed target classifier to generate meaningful counterfactuals that are suitable for providing recourse to users and improving model interpretability and fairness. The resulting method, termed \textit{CounteRGAN}, is capable of generating meaningful counterfactuals that meet or exceed prediction gain and actionability of two state-of-the-art methods while significantly improving realism and reducing latency by 2 to 7 orders of magnitude. Figure \ref{fig:countergan_arch} provides a clarifying illustration of the CounterGAN architecture applied to an example from MNIST.
The proposed technique enables providing real-time recourse to users of ML predictors deployed in a wide range of industries. The goal is to help improve the opportunity, transparency, and fairness afforded by ML predictors. The main contributions of this work include:

\begin{itemize}
\item The application of GANs to produce meaningful counterfactuals that can provide real-time recourse to users as well as improved model interpretability and fairness. 

\item Formalizing a novel \textit{Residual Generative Adversarial Network (RGAN)} that trains the generator to produce residuals that are intuitive to the notion of perturbations used in counterfactual search. This model is also shown to alleviate mode collapse.

\item The \textit{CounteRGAN} method which applies an RGAN model in conjunction with a target classifier to produce meaningful counterfactuals. It does so by 2 to 7 orders of magnitude faster than existing methods, enabling real-time applicability.

\item A CounteRGAN loss variant for when the existing classifier’s gradients or architecture is unknown (e.g., a black-box model). A proof of convergence is also provided.

\end{itemize}

%% file: related_work.tex
\section{Related Work\label{sec:related_work}}

\subsection{Counterfactuals\label{sec:counterfactuals}} 
\noindent Borrowing from philosophy and causality \cite{Lewis1973-al, pearl_causality}, counterfactuals were introduced as explanations for ML predictors by Wachter et al. \cite{Wachter2017-jr}. The authors formulated counterfactual search as a minimization problem with an added regularization term to enforce feature perturbation sparsity. Given an original data point $x$ and a ML classifier $C$, the counterfactual $x_\text{cf}$ is produced using iterations of gradient descent to increase the classifier's prediction $C_t\left(x_\text{cf}\right)$ for a given target class $t$. This approach is useful for producing counterfactuals of the desired class but tends to be slow and results may be unrealistic. 


Several approaches have targeted increasing counterfactual realism. These include a graph-based density approach \cite{Poyiadzi2019-rr} and applying an autoencoder reconstruction error term to constrict the counterfactual from straying too far from the observed feature space \cite{Dhurandhar2018-hk}. An alternative approach \cite{Mothilal2020-ge} focuses on producing multiple diverse counterfactuals for each query instance such that the user can select the most relevant. A novel technique proposed utilizing class prototypes \cite{kim_2016} to guide the counterfactual search toward high-density regions of the feature space \cite{Van_Looveren2019-hr}. 
While the aforementioned methods are limited to differentiable classifiers, a heuristic search involving "growing spheres" is used \cite{Laugel2017-tp} to produce sparse counterfactuals for non-differentiable or black-box models. This method, however, does not further address realism nor latency concerns. All of the approaches mentioned above suffer from high computational latencies. The proposed CounteRGAN method, however, is able to produce meaningful counterfactuals within real-time latency constraints for both differentiable and non-differentiable models.

Counterfactuals are also produced in adversarial perturbation techniques \cite{Goodfellow2014-yo}. For example, modifying a single pixel in an image of a horse to fool a classifier into predicting it is an image of a frog \cite{Su2017-pk}. In general, these methods are aimed at confusing a target classifier without necessarily providing meaningful recourse to users; a task that requires balancing desired prediction with realism and actionability.



\subsection{Generative Adversarial Nets (GANs)}
The introduction of GANs \cite{Goodfellow2014-wf} marked a milestone in the field of generative models. The elegance of a GAN lies in its formulation of training as an adversarial minimax game between two differentiable models able to approximate probability distributions utilizing backpropagation and gradient descent. Interest in GANs has since intensified and several novel approaches have been proposed towards improving training \cite{Salimans2016-wx, Arjovsky2017-rx} and architecture \cite{Radford2015-qx, Denton2015-ka, Zhang2016-qc}. Providing additional input such as label information to condition GANs, for example, to generate specific MNIST digits, has been previously proposed \cite{Mirza2014-kt, Odena2017-jl}. GANs have also been applied to problems that share intuitive notions with counterfactuals such as representation learning \cite{Chen2016-tt, Tran2017-vc}, image-to-image translation \cite{Isola2016-kf, Zhu2017-mz, Zhu2017-sf}, style transfer \cite{Huang2017-ii, Karras2020-uk}, and illumination \cite{Wang2017-fj, Zhang2019-kr}. The use of GANs with residual images has been proposed for attribute manipulation in images \cite{Shen2016-ge}. These methods are domain-specific and often target realism instead of reverting decisions of existing classifiers and providing actionable feedback to users. An unrelated but similarly termed "Residual GAN" \cite{Tavakolian2019-kq} uses a deep residual convolutional network to a generator to magnify subtle facial variations. In contrast, we define and use a Residual GAN, where the generator is trained to synthesize residuals directly. Unlike prior work, and to the best of our knowledge, we are the first to apply GANs towards the generation of meaningful counterfactuals for recourse.




%% file: countergan.tex
\section{GAN-based Counterfactual Generation\label{sec:countergan}}
\noindent 
To overcome the mode collapse and actionability limitations of applying standard GANs to counterfactual generation, we formalize the Residual GAN (RGAN) as a special case of GAN. 
The CounteRGAN, by contrast, is the proposed technique that couples an RGAN with a target classifier to synthesize meaningful counterfactuals.

\subsection{Residual GAN (RGAN)\label{sec:rgan}}
\noindent 
Similar to how conditional GANs \cite{Mirza2014-kt}, though initially motivated by image synthesis, have been generalized to be applicable to several domains, we also introduce a generalized RGAN formulation, whose original motivation stemmed from generating counterfactuals, but could also be applied to other domains including image synthesis and photo editing \cite{Zhang2019-kr}. The generalized RGAN is a special instance of a GAN where the generator generates residuals instead of a complete synthetic data point. As in standard GANs, a discriminator $D$ and generator $G$ are trained in a minimax game framework where the generator seeks to minimize and the discriminator aims to maximize the following value function:

\begin{equation}
\label{eq_rgan_generalized}
    \mathcal{V}_{\mathrm{RGAN}}(D, G)=\mathbb{E}_{x \sim p_{\mathrm {data}}}\log D(x)
    + \mathbb{E}_{z \sim p_{\mathrm {z}}}\log \left(1-D(z+G(z))\right),
\end{equation}
where the generator's input $z\in \mathbb Z$ is a latent variable sampled from a probability distribution $p_{\mathrm {z}}$. The input to the RGAN discriminator is $z+G(z)$, as opposed to the standard GAN which utilizes $G(z)$ directly.



The generalized RGAN formulation restricts the dimensionality of the latent (input) space to be the same as the data feature (output) space ($\mathbb Z=\mathbb{X}$).\footnote{This constraint could be overcome by utilizing an autoencoder. The synthesized data point $z+G(x)\in \mathbb{Z}$ can then be decoded to a new data point in the same space as the input data, such that $\mathrm{decoder}(z+G(x))\in \mathbb{X}$} and forces the generator to learn contingent relationships between its input and output. This constraint enables fine-grained regularization directly on the residuals \footnote{Note that the activation function for the generator's output layer constrains the residuals and therefore their impact on the final synthesized output. Thus, depending on the scenario, it is recommended to use a symmetric activation function (e.g., linear, tanh) capable of outputting positive and negative values within the same order of magnitude as the input features.} and helps to alleviate mode collapse caused when the GAN generates similar output regardless of its input which it learns to ignore. 

\subsection{CounteRGAN}
\noindent 
The proposed counterfactual search method, termed CounteRGAN, utilizes an RGAN and a fixed target classifier $C$ to produce meaningful counterfactuals for providing recourse to users and improved interpretability. The method is capable of producing counterfactuals that are of the desired target class, realistic, actionable, and require low computational latency. Below we present two variants of the CounteRGAN value function for when the classifier's gradients are and are not known. The search process seeks to maximize the value function with respect to the discriminator $D$ and minimize it with respect to the generator $G$.



If the classifier is known and differentiable, then the following CounteRGAN value function can be used:

\begin{equation*}
\label{eq_countergan_diff}
 \mathcal{V}_\text{CounterRGAN}(G, D)=\mathcal{V}_{\mathrm{RGAN}}(G, D) +\mathcal{V}_{\mathrm{CF}}(G, C, t) \\
 + \mathrm{Reg}(G(x)),
\end{equation*}
where $t$ is the target class. The first term ($\mathcal{V}_{\mathrm{RGAN}}$) uses a specialized RGAN that reads:

\begin{equation}
\label{eq_rgan_specific}
    \mathcal{V}_{\mathrm{RGAN}}(D, G)=\mathbb{E}_{x \sim p_{\mathrm {data}}}\log D(x) 
    + \mathbb{E}_{x \sim p_{\mathrm {data}}}\log \left(1-D(x+G(x))\right),
\end{equation}
where both the generator $G$ and discriminator $D$ use inputs samples $x$ from the same probability distribution $p_{\mathrm {data}}$. In isolation, this formulation would result in the generator simply learning the identity function, leading to null residuals since the inputs are already realistic data. However, since the generator is also required to account for the classifier's loss term $\mathcal{V}_{\mathrm{CF}}$, this formulation helps to enforce counterfactual realism.

The term ($\mathcal{V}_{\mathrm{CF}}$) drives the counterfactual toward the desired class $t$, it reads:

\begin{equation}
\mathcal{V}_{\mathrm{CF}}(G, C, y)=\mathbb{E}_{x \sim p_{\mathrm{data}}} \log \left(1-C_t(x+G(x))\right).
\end{equation}

The last term of the CounteRGAN value function, $\mathrm{Reg}(G(x))$, can be any weighted combination of L1 and L2 regularization terms and helps to control the sparsity and amplitude of the residuals (i.e., feature perturbations) which serves as a proxy for counterfactual actionability.


While most existing counterfactual search methods target differentiable models, the target classifiers used in production settings may often be non-differentiable or unknown (black-box).\footnote{For example, while a bank employee may have access the a loan classifier's architecture, the same cannot necessarily be said about the customer or a third-party service.} To account for such scenarios, we introduce a second CounteRGAN value function termed CounteRGAN-bb for black-box models. Instead of computing a classifier's gradients, this variant weighs the first term of the RGAN value function by the classifier's prediction score $C_t(x_i)$ such that the corresponding value function reads
\begin{equation}
\label{eq_countergan_nondiff}
    \mathcal{V}_{\mathrm{CounteRGAN-wt}}(D, G)=\frac{\sum_i C_t(x_i) \log D(x_i)}{\sum_i C_t(x_i)} 
+ \frac{1}{N} \sum_i \log \left(1-D(x_i+G(x_i))\right) + \mathrm{Reg}(G, \left\{x_i\right\}),
\end{equation}
where $\mathrm{Reg}(G, \left\{x_i\right\})$ is analogous to the regularization term introduced previously and samples $x_i$ are drawn from the entire data distribution.
$$ \mathrm{Reg}\left(G, \left\{x_i\right\}\right) = \alpha \sum_i \Vert G(x_i)\Vert_1 + \beta \sum_i \Vert G(x_i)\Vert_2^2.$$ 

The specific form of this value function is motivated by the resulting convergence properties. 
\begin{prop} 
If the discriminator is systematically allowed to reach its optimum, and the generator has sufficient capacity, then the minimax optimization of the value function from equation \ref{eq_countergan_nondiff} converges to the Nash equilibrium. The full generator's output distribution $p_{g_+}$ converges to a distribution $p_{C_t}$ defined by

\begin{equation}
p_{C_t}(x) = \mathcal N_t \; C_t(x) \; p_\mathrm{data}(x),
\end{equation}
\newline
\noindent where $N_t$ is a normalization constant.\footnote{Explicitly, $\mathcal N_t= \left(\int C_t(x) \; p_\mathrm{data}(x) \mathrm{d}x\right)^{-1}$ but it doesn't need to be computed for our purpose.} 
\end{prop}

\begin{proof}
We first introduce the full generator output function $G_+(x) = x + G(x)$, and note that the value function defined by equation \ref{eq_countergan_nondiff} can be written as 
\begin{equation}
\label{eq_countergan_nondiff_expectations}
\begin{aligned}
    &\mathcal{V}_{\mathrm{CounteRGAN-bb}}(D, G)=\mathbb{E}_{x \sim p_{C_t}} \log D(x) \\ 
    &\qquad\qquad+ \mathbb{E}_{x \sim p_{g_+}} \log \left(1-D(x)\right),
\end{aligned}
\end{equation}
since the first term on the r.h.s. of Equation \ref{eq_countergan_nondiff} is a weighted sampling estimate of $\mathbb{E}_{x \sim p_{C_t}} \log D(x)$, and for the second term, the equality $\mathbb{E}_{x \sim p_{g_+}} \log \left(1-D(x)\right)=\mathbb{E}_{x \sim p_{\mathrm{data}}} \log \left(1-D(G_+(x))\right)
$ is a consequence of the Radon–Nikodym theorem.



From the expression of the value function in equation \ref{eq_countergan_nondiff_expectations}, Proposition 1 of Goodfellow et al. \cite{Goodfellow2014-wf} implies that for any generator $G$ the optimal discriminator is 
\begin{equation}
D^*(x) = \frac{p_{C_t}(x)}{p_{g_+}(x)+p_{C_t}(x)}.
\end{equation}

The value function for an ideal discriminator thus reads:
\begin{equation}
\begin{aligned}
    &\mathcal{V}^*(G) = \mathcal{V}(D^*, G)= \mathbb E_{x\sim p_{C_t}} \log \frac{p_{C_t}(x)}{p_{g_+}(x)+p_{C_t}(x)} \\
    &\qquad\qquad + \mathbb E_{x\sim p_{g_+}} \log \frac{p_{g_+}(x)}{p_{g_+}(x)+p_{C_t}(x)}.
\end{aligned}
\end{equation}

To find the distribution $p_{g_+}^*$ that minimizes $\mathcal{V}^*$ under the probability normalization constraint, $\int p_{g_+}(x) \mathrm{d}x = 1 $, we introduce a Lagrange multiplier $\mu$. We then compute the functional derivative of $\mathcal{V}^*$ with respect to $p_{g_+}$ using the shortened notation for $p = p_{C_t}(x)$ and $q = p_{g_+}(x)$ in the following equation
\begin{equation}
\begin{aligned}
    \frac{\delta \mathcal{V}^*}{\delta q} & = \frac{\partial}{\partial q}\left[p\log\left(\frac{p}{p+q}\right) + q\log\left(\frac{q}{p+q} \right) + \mu q\right] \\
    & = \log\left(\frac{q}{p+q}\right) +\mu.
\end{aligned}
\end{equation}

The optimum of $\mathcal{V}^*$ is attained for 
\begin{equation}
    \frac{\delta V}{\delta p_{g_+}^*}(x) = 0 \quad \Longleftrightarrow \quad p_{g_+}^*(x) = \frac{p_{C_t}(x)}{\exp(\mu) - 1},
\end{equation}
from which the normalization constraint leads to
\begin{equation}
\int \frac{p_{C_t}(x)}{\exp(\mu)- 1}\mathrm d x=1 \quad \Longleftrightarrow \quad \exp(\mu)=2,
\end{equation}
such that 
\begin{equation}
p_{g_+}^*(x) = p_{C_t}(x)
\end{equation}
for all $x$. Hence $\mathcal V^*$ has a unique optimum\footnote{The optimum is a minimum here since $\mathcal V^*$ is a convex functional of $p_{g_+}$, as can be seen from the form of the second functional derivative $\frac{\delta^2 V}{(\delta p_{g_+}^*)^2}(x) = \frac{p_{C_t}(x)}{p_{g_+}(x)(p_{g_+}(x)+p_{C_t}(x))}$, which is always positive.} that is reached when 
\begin{equation}
 p_{g_+}^* = p_{C_t}.
\end{equation}

The fact that $p_{g_+}$ converges to the optimum when using the alternating gradient updates follows from Proposition 2 in \cite{Goodfellow2014-wf}.
\end{proof}

Using either value function variant, the CounteRGAN discriminator learns to discriminate between real and synthetic data points, while the generator aims to balance the desired classification with realism and sparsity (actionability) constraints. As a result, the generator learns to produce residuals that, when added to the input, produce realistic and sparse counterfactuals that are classified by $C$ to be as close to 1 for the desired class as possible. After training, the generator is able to produce counterfactuals quickly via a forward-pass through the neural network based generator.

\begin{figure*}[htb!]
     \centering
     \begin{subfigure}[t]{0.33\textwidth}
         \centering
  \includegraphics[width=0.6\columnwidth]{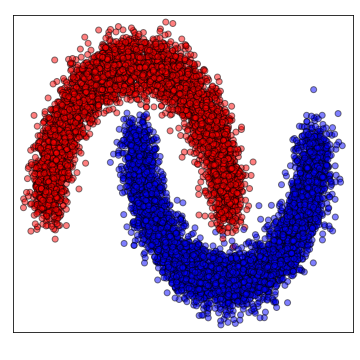} 
         \caption{Original distribution of data points.}
         \label{fig:toy_dataset_distr}
     \end{subfigure}
     \begin{subfigure}[t]{0.33\textwidth}
         \centering
  \includegraphics[width=0.6\columnwidth]{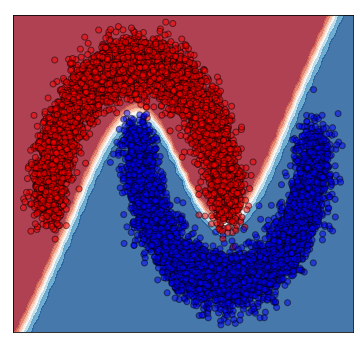} 
         \caption{Decision boundary of trained classifier.}
         \label{fig:toy_dataset_clf}
     \end{subfigure}
     \hfill
     \begin{subfigure}[t]{0.33\textwidth}
         \centering
  \includegraphics[width=0.6\columnwidth]{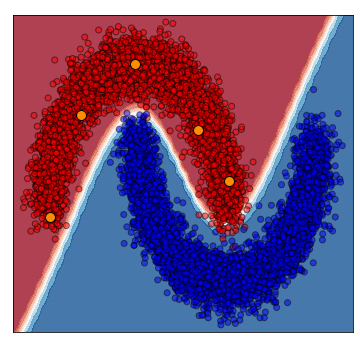} 
         \caption{Data points for counterfactuals search.}
         \label{fig:toy_dataset_samples}
     \end{subfigure}
     \begin{subfigure}[t]{0.33\textwidth}
         \centering
  \includegraphics[width=0.6\columnwidth]{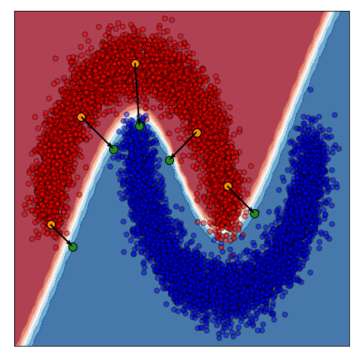} 
         \caption{Regularized gradient descent (RGD).}
         \label{fig:toy_dataset_vgd}
     \end{subfigure}
     \hfill
     \begin{subfigure}[t]{0.33\textwidth}
         \centering
\includegraphics[width=0.6\columnwidth]{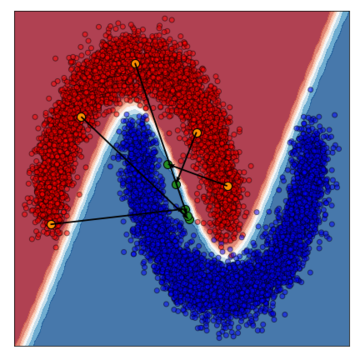}
         \caption{Standard GAN.}
         \label{fig:toy_dataset_regular_gan}
     \end{subfigure}
     \hfill
     \begin{subfigure}[t]{0.33\textwidth}
         \centering
\includegraphics[width=0.6\columnwidth]{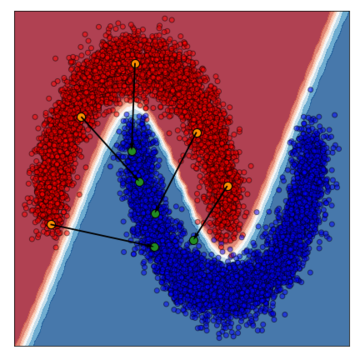}
         \caption{CounterGAN.}
         \label{fig:toy_dataset_countergan}
     \end{subfigure}
        \caption{Comparing how three different counterfactual search techniques are able to achieve their objectives while producing significantly different counterfactuals on a synthetic and binary class dataset.}
        \label{fig:toy_dataset}
\end{figure*}

\subsection{Synthetic dataset example}

Figure \ref{fig:toy_dataset} provides an example of counterfactual search using a synthetic dataset meant to illustrate the challenges faced by counterfactual generation methods. The data points shown in (a) can be interpreted as the known populations from two different societies (red/blue). An ML classifier has been trained to predict the type of society a person belongs to based on their weight (x-axis) and height (y-axis). The solid white line in (b) represents the classifier's decision boundary such that all predictions for points falling within the red shaded region are classified as persons belonging to the red society and vice-versa. The five selected orange points in (c) represent persons from the red society we seek to provide counterfactuals for. These counterfactuals should provide meaningful recourse regarding how to turn themselves into realistic looking persons of the blue society, as predicted by the classifier. The counterfactuals generated by an existing method (d) produce the correct classification result (blue) but the suggested changes would mean that the transformed individuals would not look like the rest of the known populace of the blue society (lack of realism). Using a standard GAN, the counterfactuals always result in the same or similar looking persons of the blue society. While these results are more realistic than those obtained with the previous method, the suggested changes may be harder to apply to some original persons than others (i.e., lower sparsity) and hence less actionable. The proposed CounteRGAN method (f) results in counterfactuals that are of the desired classification (blue) and are most realistic and actionable than those obtained with previous methods. Red society members seeking to imperceptibly infiltrate the blue society would benefit the most from the meaningful recourse provided by this method.

%% file: experiments.tex
\section{Experiments \label{sec:experiments}}
\begin{figure*}[hbt!]
     \centering
     \begin{subfigure}[t]{0.19\textwidth}
         \centering
         \caption{RGD}
   \includegraphics[width=0.8\columnwidth]{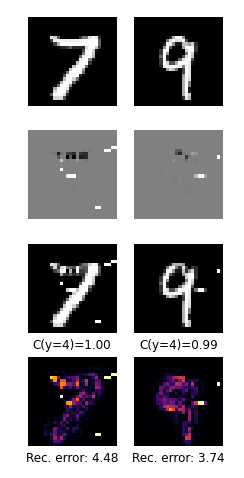} 
         \label{fig:mnist_vgd}
     \end{subfigure}
     \hfill
     \begin{subfigure}[t]{0.19\textwidth}
         \centering
         \caption{CSGP}
\includegraphics[width=0.8\columnwidth]{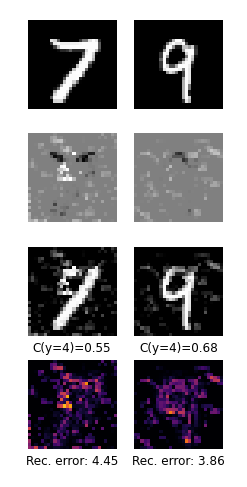}
         \label{fig:mnist_csgp}
     \end{subfigure}
     \hfill
     \begin{subfigure}[t]{0.19\textwidth}
         \centering
         \caption{Standard GAN}
\includegraphics[width=0.8\columnwidth]{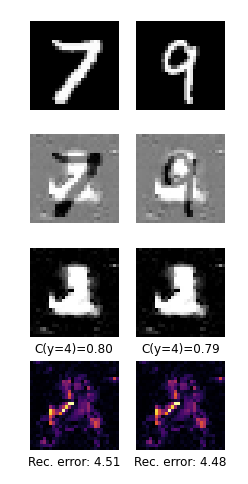}
         \label{fig:mnist_gan}
     \end{subfigure}
          \hfill
     \begin{subfigure}[t]{0.19\textwidth}
         \centering
         \caption{CounterGAN}
\includegraphics[width=0.8\columnwidth]{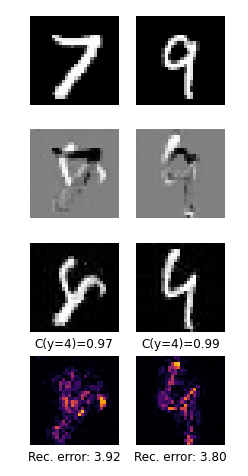}
         \label{fig:mnist_countergan}
     \end{subfigure}
          \hfill
     \begin{subfigure}[t]{0.19\textwidth}
         \centering
         \caption{CounterGAN-bb}
\includegraphics[width=0.8\columnwidth]{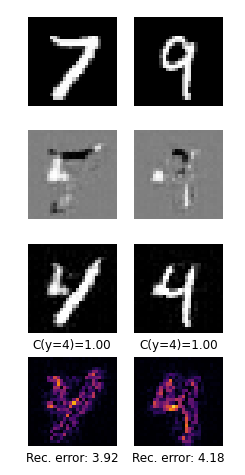}
         \label{fig:mnist_countergan_star}
     \end{subfigure}
        \caption{Comparison of counterfactual examples produced by different methods on MNIST. Given two separate digit images (7 and 9), each method is tasked with producing counterfactuals that the classifier will predict as a "4". The first row shows the original input image. The second row highlights the perturbations that the counterfactual produces (residuals in the case of CounteRGAN). Negative perturbation values are black, positive values are white, and null or zero values are grey. The third row shows the final counterfactual produced after adding the input with the perturbations. The fourth and final row displays the autoencoder reconstruction error with brighter points representing less realism. Existing methods (a) and (b) result in less realistic counterfactuals. Method (c) lacks realism as well as actionability due to mode collapse. The CounteRGAN methods (d) and (e) (black-box) result in the most realistic counterfactuals.}
        \label{fig:mnist_examples}
\end{figure*}

\noindent We compare the proposed CounteRGAN approach against two state-of-the-art counterfactual search methods \cite{Wachter2017-jr, Van_Looveren2019-hr}. As far as possible, our experiments mirror the experimental setups used in those proposals including the datasets and model architectures. 
The other counterfactual search methods mentioned in the Related Work section are not included either because they do not address realism \cite{Mothilal2020-ge, Laugel2017-tp} or because their latency is prohibitive for real-time applications \cite{Poyiadzi2019-rr}.
The first experiment is conducted using the MNIST handwritten digit dataset \cite{lecun-mnisthandwrittendigit-2010} which lends to providing visual clarity of each method's approach. The second experiment uses an Indian diabetes dataset \cite{diabetes_dataset} and helps to demonstrate that the CounteRGAN is also effective on tabular data and for diverse use cases. Lastly, the third experiment makes use of the COMPAS recidivism dataset \cite{Compas} to highlight how meaningful counterfactuals can be helpful for improving model interpretability and fairness.

\paragraph{Methods}
Given an input data point $x_i$, all methods described below aim to produce a counterfactual $x^\text{cf}_i$ that a target classifier $C$ will predict as the desired class.

\begin{itemize}
    \item \textit{Regularized Gradient Descent (RGD)}: a gradient descent based counterfactual search \cite{Wachter2017-jr} that minimizes the sums of the squared differences between the desired outcome and the counterfactual. A regularization term is used to enforce sparsity. \footnote{For this method and the next, we use the implementations (including gradient approximating versions for black-box models) provided by \url{https://github.com/SeldonIO/alibi}.}
    \item \textit{Counterfactual Search Guided by Prototypes (CSGP)}: this method \cite{Van_Looveren2019-hr} extends RGD by using class prototypes to push the counterfactual towards a more realistic data point of the desired class. The value function is modified to include a distance measure from the counterfactual to the class prototype in latent space ($L_\text{proto}$).
    \item \textit{Standard GAN (GAN)}: This method applies a standard GAN \cite{Goodfellow2014-wf}, in conjunction with the target classifier $C$. The generator is modified to use real data points as input (as opposed to random latent variables) and synthesize complete counterfactuals.
    \item \textit{CounteRGAN}: The proposed method from section \ref{sec:countergan} that uses the specialized RGAN together with the target classifier $C$. The value function from Equation \ref{eq_countergan_diff} is used when $C$ is a white-box model (i.e., known gradients) and Equation \ref{eq_countergan_nondiff} is used when $C$ is a black-box model (i.e., unknown or undefined gradients).
\end{itemize}

\renewcommand{\arraystretch}{1.4}
\begin{table}
\centering
\scriptsize
\begin{tabular}{ c | c  }
Metric & Formula \\
 \hline
 Counterfactual prediction gain & $\mathbb{E}\left[ C(x^{\text{cf}}_i) - C(x_i)\right]$ \\  
 Realism & $\mathbb{E}\left[  \left\Vert\text{AE}\left( x^{\text{cf}}_i \right) - x^{\text{cf}}_i\right\Vert_2^2\right]$ \\ 
 Actionability (Sparsity \& proximity) & $\mathbb{E}\left[ \left\Vert x^{\text{cf}}_i - x_i\right\Vert_1 \right]$\\
 Latency & $\mathbb{E}\left[\delta t_i \right]$\\
\end{tabular}
\medskip
\caption{Evaluation metrics summary. $C$ is the target classifier and $x_i$ denotes the data point for which a counterfactual ($x_i^\text{cf}$) is sought. An autoencoder ($\text{AE}$) is used to reconstruct $x_i^\text{cf}$. Expectations are computed over the test sets.}
\label{table:eval_metrics}
\end{table}

\paragraph{Evaluation metrics}

\renewcommand{\arraystretch}{1.6}
\begin{table*}[hbt!]
\scriptsize
\centering
\begin{tabular}{l||c|c|c|c||c|c|c}
 \multirow{2}{*}{} &           \multicolumn{4}{c||}{White-box classifier} & \multicolumn{3}{c}{Black-box classifier} \\ 
&                    RGD &                        CSGP &                          GAN &                  CounterGAN &    RGD & CSGP & CounterRGAN \\
\hline
$\uparrow$ Prediction gain & \textbf{0.83 $\pm$ 0.01} & 0.43 $\pm$ 0.00 & 0.69 $\pm$ 0.01 & 0.80 $\pm$ 0.01 & 0.45 $\pm$ 0.01 & 0.41 $\pm$ 0.00 & \textbf{0.85 $\pm$ 0.01} \\
$\downarrow$  Realism    & 4.56 $\pm$ 0.01   & 4.58 $\pm$ 0.01   & 4.50 $\pm$ 0.00             & \textbf{3.95 $\pm$ 0.01} & 3.94 $\pm$ 0.01 &  \textbf{3.58 $\pm$ 0.01} & 4.37 $\pm$ 0.01 \\
$\downarrow$  Actionability & \textbf{20.63 $\pm$ 0.41} & 54.24 $\pm$ 0.60 & 151.98 $\pm$ 0.43          & 79.47 $\pm$ 0.47 & \textbf{31.86 $\pm$ 0.61} & 48.79 $\pm$ 1.69 & 72.99 $\pm$ 0.52 \\
$\downarrow$  Latency (ms)  & 4,129.57 $\pm$ 3.33     & 5,359.58 $\pm$ 2.72    & \textbf{13.05 $\pm$ 0.04} & 13.33 $\pm$ 0.04     & 8,464.10 $\pm$ 42.54 &   30,235.47 $\pm$ 553.47  &  \textbf{13.52 $\pm$ 0.04} \\
$\downarrow$  Batch latency (seconds) & 4,129,570    & 5,359,580    & \textbf{45} & \textbf{45} & 84,641,012 & 302,354,681 &  \textbf{45} \\
\end{tabular}
 \medskip
\caption{MNIST test data results (mean and 95\% confidence interval). The arrows indicate whether larger $\uparrow$ or lower $\downarrow$ values are better, and the best results are in bold. The realism metric typically ranges from 3.89 (mean reconstruction error on the test set) to 11.99 (reconstruction error random uniform noise $[0, 1]$). Computations are performed using the entire test set (10,000 samples).}
\label{table:mnist_metrics}
\end{table*}

To evaluate the relative performance of the methods, we identify four desirable properties of counterfactual generation and propose the corresponding metrics detailed below (see Table \ref{table:eval_metrics} for a summary). All metric results from the experiments, except for batch latency, are based on averages of individually computed counterfactuals using the test data. Batch latency is the total computation time necessary to produce counterfactuals for an entire batch. Each table presents the results of the methods assuming that the target classifier's gradients are known (white-box model) or unknown (black-box model). 
\begin{itemize}
    \item \textit{Prediction gain}: the difference between the classifier's prediction on the counterfactual ($C_t(x^{\text{cf}}_i)$) and the input data point ($C_t(x_i)$), for the target class $t$. Since the maximum score classifier $C$ can predict is 1, the range for prediction gain is $[0,1]$ with higher gain indicating more improvement.
    \item \textit{Realism}: a measure of how well a counterfactual "fits in" with the known data distribution. We adopt a strategy inspired by \cite{Van_Looveren2019-hr, Dhurandhar2018-hk}, in which we train a denoising autoencoder $\text{AE}\left( \cdot\right)$ on the training set and use the L2 norm of the reconstruction error as a measure of realism. A lower value represents higher realism. 
    \item \textit{Actionability (sparsity \& proximity)}: a measure of the number and magnitude of perturbations present in the counterfactual ($x^{\text{cf}}_i$) relative to the input data point ($x_i$) using the L1 norm. A lower value corresponds to fewer changes and hence a higher degree of actionable feedback. Sparsity and proximity are commonly used \cite{Mothilal2020-ge}, albeit imperfect, proxies for true actionability which is inherently difficult to quantify and a promising area for future work.
    \item \textit{Latency}: the computational latency needed to generate counterfactuals. Individual counterfactual computations can impact real-time applicability. Batch results are useful to highlight scalability limitations since large amounts of counterfactuals may be desired to be generated without real-time constraints but within practical latency and cost budgets. Lower values are better and subsecond latencies are necessary for real-time applicability.
\end{itemize}

\vspace{-0.28cm}
\subsection{First Experiment: MNIST image dataset}
MNIST consists of 70,000 images of handwritten digits (28x28 black and white pixels, that we normalize to have values between 0 and 1) with equal amounts of samples for each digit class. The images are split for training and testing with 60,000 and 10,000 samples respectively, both of which are balanced in terms of labels.

A convolutional neural network (CNN) is used as the target classifier which is trained to correctly classify the digits (98.6\% accuracy on the test set). In addition to the classifier, we train a denoising convolutional autoencoder that is used to gauge counterfactual realism. Each method is tasked with generating counterfactuals that the classifier should predict as a "4" digit. All results are based on the averages from generating counterfactuals for all of the 10,000 samples from the test set. 


Examples of counterfactuals for two digits are shown in Figure \ref{fig:mnist_examples}. For those examples, all methods successfully produce counterfactuals that are labeled as "4" by the classifier, with predicted probabilities ranging from 0.55 to 1. The RGD method (Figure \ref{fig:mnist_vgd}) suggests counterfactuals that are more similar to adversarial attacks in the sense that they consist of subtle perturbations that lead to the desired classification but are highly unrealistic. The CSGP algorithm (Figure \ref{fig:mnist_csgp}) seems to perform better visually, affecting relevant pixels to turn the digits into the desired "4" but still lacks realism. The counterfactual search with a regular GAN (Figure \ref{fig:mnist_gan}) saliently exhibits mode collapse. Without the residual formulation, the generator simply learns to generate the same image regardless of the input. The two CounterGAN formulations (Figures \ref{fig:mnist_countergan} and \ref{fig:mnist_countergan_star}) output visually convincing counterfactuals, as corroborated by the large classifier scores (0.97 to 1) and low autoencoder reconstruction errors.

The complete metrics results for the MNIST dataset are presented in Table \ref{table:mnist_metrics}. While all methods largely increase the prediction of the target class, CSGP is noticeably less impactful. The RGD method outputs sparser counterfactuals at the significant cost of realism. The two CounterGAN variants, by contrast, generate the most realistic counterfactuals with high actionability and prediction gain. Notably, the GAN and proposed CounteRGAN approaches also achieve >300x and >600x latency improvements over existing methods when generating single counterfactuals on white-box and black-box classifiers respectively. On a batch of the full 10000 samples from the test set, the GAN based methods achieve an impressive 5 to 7 orders of magnitude improvement. 

Although we have not included additional image datasets in our experiments, the same methodology can be applied to large-scale image datasets such as ImageNet \cite{deng2009imagenet} or CelebA \cite{liu2015faceattributes} with no modification. While the CounteRGAN setup and training principles would remain the same, the extension to larger datasets would require more complex generator and discriminator architectures, and a more extensive optimization of the training parameters. Because this work focuses on actionable feedback rather than the generation of realistic images, we leave applications of the CounteRGAN to large-scale datasets for future work and turn our attention to tabular data in the following.


\renewcommand{\arraystretch}{1.6}
\begin{table*}[hbt!]
\scriptsize
\centering
\begin{tabular}{l||c|c|c|c||c|c|c}
 \multirow{2}{*}{} &           \multicolumn{4}{c||}{White-box classifier} & \multicolumn{3}{c}{Black-box classifier} \\ 
&                    RGD &                        CSGP &                          GAN &                  CounterGAN &    RGD & CSGP & CounterRGAN \\
\hline
$\uparrow$ Prediction gain &      0.15 $\pm$ 0.01 &           0.13 $\pm$ 0.02 &            0.15 $\pm$ 0.03 &  \textbf{0.33 $\pm$ 0.04} &  \textbf{0.17 $\pm$ 0.00} & 0.13 $\pm$ 0.00 & \textbf{0.16 $\pm$ 0.02} \\
$\downarrow$  Realism      &      2.20 $\pm$ 0.24 &           2.03 $\pm$ 0.11 &            3.33 $\pm$ 0.11 &  \textbf{1.79 $\pm$ 0.11} &  2.22 $\pm$ 0.01  & \textbf{1.98 $\pm$ 0.01} & 2.13 $\pm$ 0.12 \\
$\downarrow$  Actionability  &      1.64 $\pm$ 0.20 &  \textbf{1.14 $\pm$ 0.19} &            9.46 $\pm$ 0.53 &           6.91 $\pm$ 0.43 & 1.75 $\pm$ 0.02  & \textbf{1.29 $\pm$ 0.02} & 2.97 $\pm$ 0.12  \\
$\downarrow$  Latency (ms) &  1,195.91 $\pm$ 5.65 &       3,211.67 $\pm$ 11.65 &  1.68 $\pm$ 0.06 &          \textbf{1.51 $\pm$ 0.03} & 2,525.99 $\pm$ 1.23  & 15,921 $\pm$ 23.66 &  \textbf{1.82 $\pm$ 0.12} \\
$\downarrow$  Batch latency (seconds) & 204.58 & 483.88 & 0.26 & \textbf{0.23} & 453.45  & 2,228.23 & \textbf{0.32}
\end{tabular}
\medskip
\caption{Diabetes test data results (mean and 95\% confidence interval). The arrows indicate whether larger $\uparrow$ or lower $\downarrow$ values are better, and the best results are in bold. The realism metric typically ranges from 1.844  (mean reconstruction error on the test set) to 2.443 (reconstruction error on random Gaussian noise). Computations are performed using the entire test set (154 samples).}
\label{table:diabetes_metrics}
\end{table*}

\vspace{-0.28cm}
\subsection{Second experiment: Pima Indians Diabetes dataset}
Following the experiments in Wachter et al. \cite{Wachter2017-jr}, we utilize the Pima Indians Diabetes dataset \cite{diabetes_dataset}. It is composed of low dimensional tabular data and helps to validate the CounteRGAN's versatility and its applicability to diverse use cases. The dataset contains 8 features describing the relevant characteristics of patients useful for predicting diabetes. The target label is positive if the patient has diabetes (268 examples) and negative otherwise (500 examples). We use stratified (label balanced) sampling with 80\% of the dataset being assigned to the train set and the remaining 20\% for the test set. The classifier is the same as the neural network architecture used in \cite{Wachter2017-jr} and achieves an accuracy of 74.68\% on the test set.

For this experiment we introduce the important concept of \textit{mutable} and \textit{immutable features}. For most practical applications of counterfactual search, certain features may be hard or impossible to change and can be considered immutable. Though features typically vary in their degree of mutability, for the purposes of this experiment we consider features as either mutable or immutable. For the Pima Indians Diabetes dataset, we consider \textit{Pregnancies}, \textit{Age}, and \textit{Diabetes Pedigree Function} features to be immutable. We use \textit{Glucose}, \textit{Insulin}, \textit{Body Mass Index}, \textit{Tricept Skin Fold Thickness}, and \textit{Blood Pressure} as mutable features. In practice, we apply counterfactual search with no modifications, then simply cancel the perturbations applied to immutable features.




Table \ref{table:diabetes_metrics} summarizes our findings for this experiment. On this dataset, all methods appear equally capable of improving classifier prediction gain. The CounterGAN generates more realistic instances, and the CSGP outputs the sparsest counterfactuals. Even on this low-dimensional dataset, the CounteRGAN is able to meet or exceed the evaluation metrics of counterfactuals produced by existing methods while heavily outperforming them in terms of latency. This includes >1,000x to >2,000x improvements for individual counterfactuals on white-box and black-box models respectively and from 3 to 4 orders of magnitude for batch generation of all counterfactuals.  

The evaluation results validate that the proposed CounteRGAN method is capable of overcoming the main limitations of existing methods, namely the lack of realism and high latency. It also provides similar or better prediction gain and actionability on high dimensional images and a low-dimensional tabular dataset. The impressive latency improvements are pivotal with regard to real-time applicability and scalability. This is due to the generator only needing a forward-pass through the neural network as opposed to performing a new counterfactual search for every data point, as required by existing methods. 


\renewcommand{\arraystretch}{1.6}
\begin{table*}[hbt!]
\scriptsize
\centering
\begin{tabular}{l||c|c|c|c||c|c|c}
 \multirow{2}{*}{} &           \multicolumn{4}{c||}{White-box classifier} & \multicolumn{3}{c}{Black-box classifier} \\ 
&                    RGD &                        CSGP &                          GAN &                  CounterGAN &    RGD & CSGP & CounterRGAN \\
\hline
↑ Prediction gain    &     \textbf{0.38 ± 0.01} &     0.06 ± 0.01 &  0.29 ± 0.01 &  0.07 ± 0.01 & \textbf{0.38 ± 0.01} & 0.06 ± 0.01 & 0.12 ± 0.01 \\
↓ Realism & 1.60 ± 0.08 & 0.78 ± 0.06 & \textbf{0.57 ± 0.00} & 0.85 ± 0.09 & 1.60 ± 0.08 & \textbf{0.77  ± 0.06} &0.93 ± 0.09  \\
↓ Sparsity           &     2.07 ± 0.05 &     \textbf{0.53 ± 0.08} &  7.32 ± 0.16 &  0.85 ± 0.05 & 2.07 ± 0.05 & \textbf{0.50 ± 0.08} & 1.48 ± 0.08 \\
↓ Latency (ms)       &  1,704.62 ± 2.12 &  3,312.14 ± 5.46 &  \textbf{1.39 ± 0.01} &  1.43 ± 0.01 & 3,005.13 ± 2.35 & 9,894.08 ± 51.75 &  \textbf{1.42 ± 0.12} \\
↓ Batch latency (s) &       2,459.76 &       4,779.42 &       \textbf{2.00} &       2.06 &  4,336.40 & 14,277.15 &     \textbf{2.04} \\

\end{tabular}
\medskip
\caption{COMPAS test data results (mean and 95\% confidence interval). The arrows indicate whether larger $\uparrow$ or lower $\downarrow$ values are better, and the best results are in bold. The realism metric typically ranges from 0.87 (mean reconstruction error on the test set) to 5.43 (reconstruction error of random uniform noise in $[0, 1]$).}
\label{table:compas_metrics}
\end{table*}

\subsection{Third experiment: COMPAS recidivism dataset}



Certain applications of predictive models can have permanent life-altering consequences for individuals.
Within criminal justice systems, for instance,
recidivism prediction models such as the COMPAS score \cite{Compas} are consulted to guide criminal sentencing in several states and major cities in the United States \cite{WeaponsMathDestruction}. This experiment showcases how meaningful counterfactuals can be applied towards improving model interpretability and fairness by exposing biases, including racial and gender biases which are harmful and pervasive in our societies.



\renewcommand{\arraystretch}{1.2}
\begin{table*}
\scriptsize
\centering
\begin{tabular}{l||c||c|c|c|c||c|c|c}
 \multirow{2}{*}{} &    \multirow{2}{*}{Initial values}    &    \multicolumn{4}{c||}{White-box classifier} & \multicolumn{3}{c}{Black-box classifier} \\ 
&           &          RGD &                        CSGP &                          GAN &                  CounterGAN &    RGD & CSGP & CounterRGAN \\
\hline
age                                         &             24 &     - &   +1 &         +23 &         +6 &     - &  +2 &       +12 \\
priors\_count                                &              3 &    -9 &   -1 &          -4 &         -2 &   -9 & -1    &      -1 \\
days\_b\_screening\_arrest                     &             -1 &    -1 &    - &          -3 &          - &      -1  & - &   -12 \\
sex\_Male                                    &              1 &     - &    - &          -1 &          - &   - & - &          - \\
c\_charge\_degree\_M                           &              0 &     - &    - &          +1 &          - &       - & - &      - \\
c\_charge\_desc\_Pos Cannabis W/Intent Sel/Del &              1 &     - &    - &          -1 &          - & - & - &           -1 \\
c\_charge\_desc\_Possession of Cocaine         &              0 &     - &    - &           - &          - &        - & - &    +1 \\
race\_Caucasian                              &              0 &     - &    - &          +1 &          - &   - & - &         +1 \\
\hline
Classifier Prediction (prob of not recidivating)                      &           0.36 &  0.99 &  0.50 &        0.87 &       0.71 &    0.99 & 0.52 &      0.54 \\
\end{tabular}
\medskip
\caption{Comparison of counterfactual examples produced by different methods given a sample data point from the COMPAS recidivism dataset. The leftmost column shows the features that are perturbed by at least one of the evaluated counterfactual search method. The bottom row shows the likelihood of not recidivating, as predicted by the classifier. Note that some of the counterfactuals suggest changing the race and gender to alter the prediction indicating that the classifier and dataset holds potentially unfair biases.}
\label{table:compas_example}
\end{table*}

The COMPAS dataset consists of 7,214 arrests logged in Broward County, Florida, and contains 29 features describing the demographics and criminal history of the defendants. The binary target label is positive if the defendant did not re-offend within two years after the arrest (55\% of the data) and negative if they did (45\% of the data). Numerical features are standardized and categorical variables are one-hot-encoded. The one-hot-encoded features are then perturbed in the same fashion as the numerical features and then rounded to the closest binary value for the final residuals. \footnote{An alternative approach would be to handle categorical features using pairwise distance measures and multi-dimensional scaling \cite{Van_Looveren2019-hr}.}
Since the main objective of this scenario is aimed at model interpretability and fairness, we do not differentiate between mutable and immutable features. We randomly assign 80\% of samples to the train set and the remaining 20\% to the test set. A neural network with three hidden layers is trained and reaches an accuracy of 69.72\% on the test set. A threshold of 0.5 is chosen for determining whether an individual will recidivate ($<$0.5) or not ($\ge$0.5).

Table \ref{table:compas_metrics} presents the results for the counterfactual search methods on the COMPAS test set. Similar to previous experiments, the RGD approach tends to produce unrealistic counterfactuals with large increases to the classifier's prediction. Conversely, CSGP typically leads to small increases of the classifier score but outputs sparser and more realistic perturbations. The regular GAN method achieves decent gains in prediction score and realism but suffers greatly with respect to sparsity and hence actionability. The CounteRGAN methods proposed in this work are more satisfying than RGD in terms of realism and sparsity. They also achieve similar increases of the classifier prediction as CSGP and produce counterfactuals >1,000x faster than RGD and CSGP.

Specific examples are relevant for investigating what, if any, biases a classifier has learned. Table \ref{table:compas_example} presents one such data point from the test set. It compares the original feature values with those from counterfactuals produced by every method included in our benchmark. Each method is able to generate a counterfactual that successfully reverts the prediction although they propose very different perturbations to the features. RGD suggests an unrealistic change that corresponds to a negative number of prior offenses. CSGP is able to barely flip the prediction (score $\ge$0.5) with minimal and realistic changes. Though general conclusions should be based on subsequent analysis of complete datasets, counterfactuals such as these can help to probe a classifier's decision boundary in the vicinity of individual data points. By illuminating regions in the feature space where the classifier predicts non-recidivism, they can serve a pertinent role in understanding the impact and relation certain feature value changes will have on the final prediction, thereby adding to a model's interpretability. 

Interestingly, the counterfactuals produced by the GAN and CounteRGAN methods for black-box classifiers find that changing the race to "Caucasian" instead of "Black" contributes to reversing the prediction. In addition, the GAN counterfactual also suggests changing the gender from "Male" to "Female". These insights signal that the recidivism predictor likely holds unfair biases. By extension, these biases can also be manifest in the COMPAS dataset. This is not necessarily certain, however, since it may have been by chance that the training subset was unbalanced and the model simply picked up on these spurious biases. Insights such as these illustrate the potential counterfactuals have in helping to audit the fairness of ML systems which should be of paramount relevance to all practitioners.

%% file: discussion.tex
\section{Discussion \label{sec:discussion}}
 \noindent The proposed CounteRGAN approach applies a novel Residual GAN (RGAN) together with a fixed target classifier to produce realistic and actionable counterfactuals that achieve favorable prediction increases at low fixed latencies. Separately defined value functions allow the CounteRGAN to work effectively even when the target classifier is non-differentiable or a black-box model. Experiments on three diverse datasets show that the CounteRGAN is able to generate counterfactuals at 2 to 7 orders of magnitude faster than two state-of-the-art methods. The drop from requiring seconds to milliseconds opens the door to real-time applicability. The resulting counterfactuals are generally more realistic than competing methods while matching or exceeding prediction gain and actionability. In use cases such as criminal justice which can have pivotal consequences for users, this approach has also shown encouraging promise for probing a classifier's decision boundaries and highlighting potentially unfair biases. Meaningful counterfactuals, such as those produced using the CounteRGAN method, can provide real-time recourse to users and help improve model interpretability and fairness. Together, these form the critical foundations for building effective, scalable, and trustworthy ML systems.

Several promising areas outside the scope of this work are left for future research. These include investigating additional techniques to quantify and ensure actionability, applying the RGAN and CounteRGAN to additional domains, improving the training and architecture, addressing partially mutable features, and experimenting with iteratively improving the counterfactuals by creating a feedback loop to the generator.

%% file: ethical_statement.tex
\paragraph{Ethical Statement}
\noindent 
Realistic and actionable counterfactuals, as provided by the CounteRGAN, are able to suggest meaningful recourse to users seeking to understand how to improve relevant  outcomes decided by ML predictions. This, in effect helps to provide a measure of transparency and opportunity to users. This can have positive and even life-altering effects since many of the targeted predictors are found in essential industries such as healthcare, finance, human resources, and criminal justice systems. Counterfactuals can also flag problematic biases of models used in production, helping to ensure fairness.
Similar to adversarial attacks, however, this method could be utilized for malicious or nefarious ends. For example, to produce realistic misinformation capable of fooling detectors on social media sites or aiding in financial fraud. 
The social and ethical impact of this work, therefore, are of great potential. In order to avoid detrimental behaviours, however, practitioners should exercise caution with respect to what recourses are shared with the users.

